\documentclass[letterpaper, 10 pt, conference]{ieeeconf}  

\IEEEoverridecommandlockouts                              

\usepackage{subfigure}
\usepackage{times}

\usepackage{amssymb,url,amsmath,amsthm}
\usepackage{algorithmic,algorithm}
\usepackage{multirow}
\usepackage{mathabx}
\usepackage{graphicx}
\usepackage{color}

\usepackage{multicol}
\usepackage[bookmarks=true]{hyperref}

\newtheorem{theorem}{Theorem}[section]

\newtheorem{lemma}[theorem]{Lemma}
\newtheorem{claim}[theorem]{Claim}

\newtheorem{definition}{Definition}[section]

\begin{document}

\title{Achieving Multi-Tasking Robots in \\ Multi-Robot Tasks}

\title{\LARGE \bf
Achieving Multi-Tasking Robots in  Multi-Robot Tasks
}


\author{Yu Zhang and Winston Smith
\thanks{Yu Zhang and Winston Smith are with the School of Computing, Informatics and Decision Systems Engineering,
        Arizona State University
        {\tt\small \{yzhan442, wtsmith7\}@asu.edu}.}%
}



%

\maketitle
\thispagestyle{empty}
\pagestyle{empty}

\begin{abstract}

One simplifying assumption made in distributed robot systems 
is that the robots are single-tasking: each robot operates on a single task at any time. 
While such a sanguine assumption is innocent to make in situations with sufficient resources so that the robots can operate independently,
it becomes impractical when they must share their capabilities.
In this paper, we consider multi-tasking robots with multi-robot tasks.
Given a set of tasks, each achievable by a coalition of robots,
our approach allows the coalitions to overlap and task synergies to be exploited
by reasoning about the physical constraints that can be synergistically satisfied for achieving the tasks. 
The key contribution of this work 
is a general and flexible framework to achieve this ability
for multi-robot systems 
in resource-constrained situations to extend their capabilities. 
The proposed approach is built on the information invariant theory, 
which specifies the interactions between information requirements. 
In our work, we map physical constraints to information requirements,
thereby allowing task synergies to be identified via the information invariant framework.
We show that our algorithm is sound and complete under a problem setting with multi-tasking robots. 
Simulation results show its effectiveness under resource-constrained situations and in handling challenging situations in a multi-UAV simulator.

\end{abstract}

\section{Introduction}
\label{sec:introduction}

To address a multi-robot task, one simplifying assumption made in the literature
is that the robots are single-tasking. 
This assumption, however, is impractical in situations where the robots must coordinate closely to share their capabilities,
such as when a robot has a capability that is required in multiple tasks.
To handle such resource-constrained situations,
a simple solution is to achieve the tasks sequentially. 
Unfortunately, such a solution, besides having a negative impact on task efficiency, is feasible only when no concurrent execution is required among the tasks for which capabilities must be shared.
As a result, it significantly limits the capabilities of multi-robot systems. 
Instead, in this paper, we consider multi-tasking robots in multi-robot tasks
to enable a robot to operate on multiple tasks at the same time.

While multi-tasking robots are desirable when there are resource contentions,
the fundamental question regarding its feasibility must be carefully considered. In particular, it is affected heavily by the compatibility of the physical constraints to be satisfied for achieving the tasks.\footnote{While there may be other factors that affect the feasibility, such as limitation on the communication bandwidth, our focus is on physical constraints as the influences of other factors are usually less direct or critical.}
For example, for a robot to share its localization capability, 
it must stay within the proximity of the robot that requires its assistance;
for a UAV to share its camera sensor
with a ground vehicle that is assigned to a monitoring task, it must maintain its camera head direction towards the target. 
Hence, the main challenge to enable multi-tasking robots lies in identifying synergies between the underlying physical constraints for the tasks.

The proposed approach is built on the information invariant theory~\cite{donald1995information} that specifies interactions between information requirements. 
In our work, the physical constraints to be satisfied are mapped to information instances,
which are categorized by information types.
Fig. \ref{fig:demo} illustrates a scenario with two tasks:
one of them is a centroid task that requires the three robots to maintain their centroid over the base station,
which is specified as a constraint on the centroid information over the three robots;
the other one is a monitoring task that requires one of the robots to maintain a target within its observation range,
which is specified as a constraint on the relative position information from the monitoring robot to the target. 
The interactions between the constraints can then be described as information interactions, which allows 
us to identify task synergies using the information invariant framework. 
More specifically, given a set of 
physical constraints for the tasks, 
our approach checks if they are compatible 
according to a set of rules that 
govern information interactions. 
A task synergy is identified, for example, when the constraints for different tasks can be satisfied simultaneously even when there are shared resources (i.e., robots) among them.
Fig. \ref{fig:demo} illustrates a synergy between the two tasks.

 \begin{figure}[t!]
    \centering
    \begin{subfigure}
        \centering
        \includegraphics[width=0.41\columnwidth]{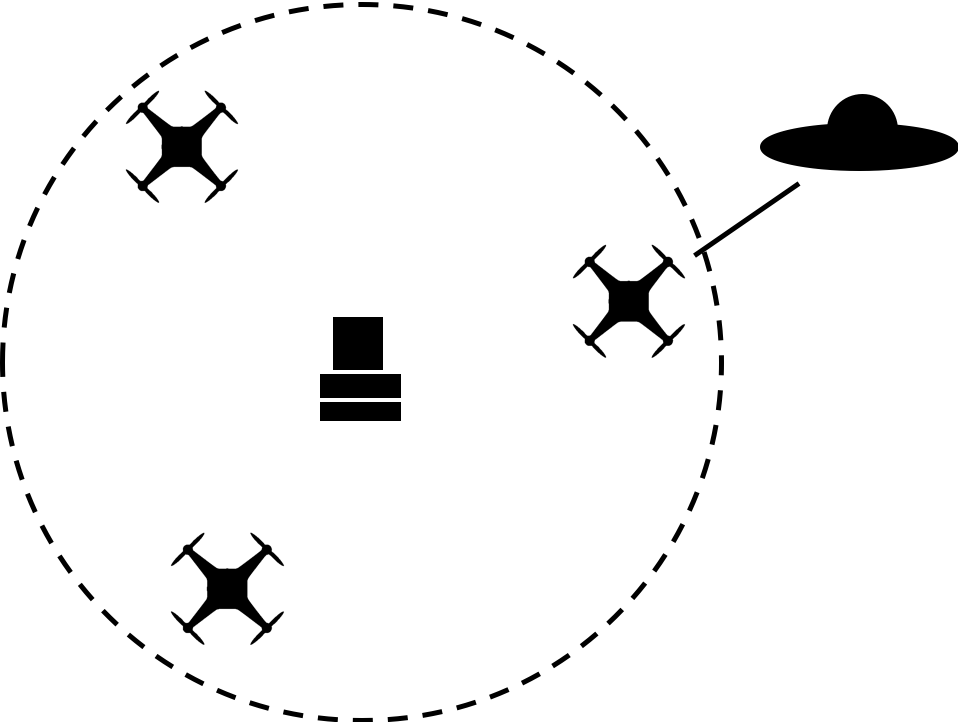}
    \end{subfigure}%
    ~ 
    \begin{subfigure}
        \centering
        \includegraphics[width=0.3\columnwidth]{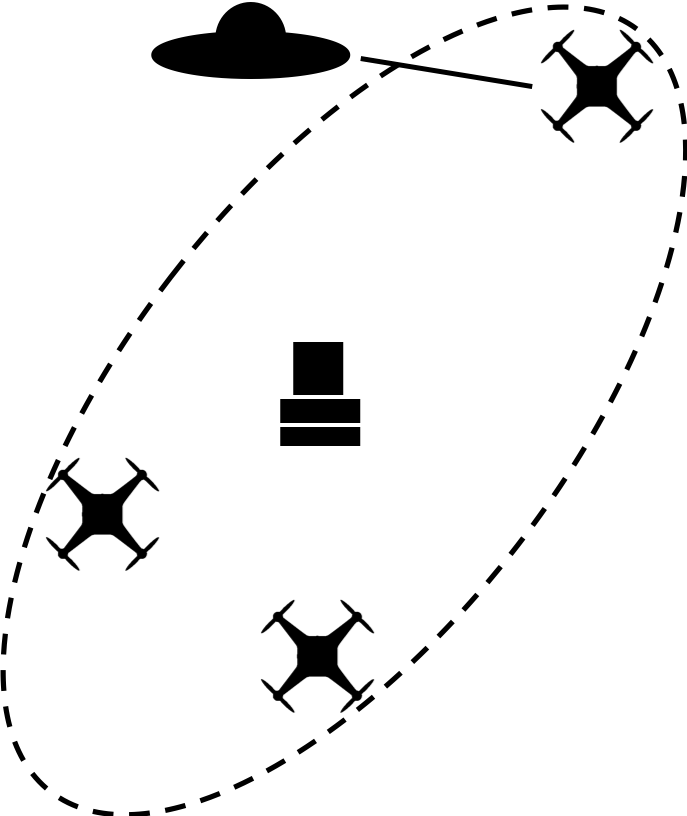}
    \end{subfigure}
    \caption{Scenario that illustrates a synergy between two tasks, where one robot is shared: a centroid task and a monitoring task. The left and right figures show the change of robot configurations as the target changes its position.}
    \vskip-15pt
    \label{fig:demo}
\end{figure}

To the best of our knowledge, our work represents the first general framework for achieving multi-tasking robots in multi-robot tasks~\cite{gerkey}.
It removes a restrictive assumption made in multi-robot systems
by enabling overlapping coalitions.
A formal framework is presented based on the information invariant theory. 
Simulation results show that our approach not only achieves better efficiency but also extends the capabilities of such systems in a multi-UAV simulator,
especially in situations where resources are limited.

\section{Related Work}

Much effort in multi-robot systems has been dedicated to the task allocation problem~\cite{gerkey}. 
When the robots are single-tasking and tasks are single-robot tasks, 
the problem is also known as the assignment problem, 
which has efficient solutions~\cite{kuhn1955hungarian, Liu1}.
Addressing multi-robot tasks, however, significantly increases the problem complexity~\cite{gerkey},
for which there has been a noteworthy amount of focus on approximate solutions~\cite{shehory, vig, sandholm, adams2011coalition, zhang2013considering},
and realistic concerns while implementing task allocation with distributed robot systems~\cite{arkin2, botelho, dias, fua, gerkey2, kalra, parker2, vig, yu2, yu3, zlot}.
As far as we know, there exists little discussion on generalizing the problem
to multi-tasking robots with instantaneous assignments~\cite{gerkey}.\footnote{Note that our work differs from the task scheduling problem (i.e., time extended assignments) where a robot can be assigned to different tasks at different times. Our problem setting requires the tasks to be assigned and executed at the same time.} 
Although the focus here is not on task allocation, 
it is an interesting direction to study how to integrate our method with task allocation algorithms. 

The ability to achieve multi-tasking robots may appear similar to the existing notions of task synergy~\cite{liemhetcharat2014weighted, parker1999cooperative}
 and overlapping coalition structure~\cite{shehory1996formation, dang2006overlapping}.
However, these prior methods concern mainly with the optimization problem of utility maximization with the influence between the assignments given a priori, i.e., how the assignment of a coalition to a task may contribute to the overall utility given the other assignments.
Such influence is often assumed to be captured by a pre-specified heuristic function
that is difficult to compute in real-world settings.
For one, modeling task synergy involves complex reasoning
about the physical constraints as we argued. 
Our study goes beyond prior work by explicitly modeling the influence of overlapping coalitions on task feasibility due to the consideration of physical constraints, 
and thus addresses the fundamental question of whether task synergies are present.  
A similar observation was made in ~\cite{vig} about the impact of physical constraints on task allocation and the issue was addressed by manually specifying a feasibility function.
A general framework that can automatically reason over the  constraint space is missing.

The information invariant theory is introduced to capture the equivalence of sensori-computational systems~\cite{donald1995information}. It has since then been used to develop systems that have demonstrated an impressive level of flexibility~\cite{yu3, tang2005asymtre}. Information invariant is well connected to the notion of information space~\cite{lavalle2006planning},
where both must reason about the relationships between different information requirements. 
The difference being that the latter is often focused on the minimalistic aspect~\cite{tovar2004gap}. 
To the best of our knowledge, however, we are the first to apply the information invariant theory to reason about synergies between multi-robot tasks.

\section{Approach}
\label{sec:overview}

A physical constraint in our work is modeled as an information instance or a combination of information instances,
which are categorized by information types. 
Next, we first formally introduce these two notions. 

\begin{definition}[Information Instance]
    An information instance, denoted as $F(\textbf{E})$, captures the {\textbf{\textit{semantics}}} of information where $F$ is the information type and $\textbf{E}$ is an ordered set of referents.
    \label{def:info-inst}
    \qed
\end{definition}

\begin{table*}
    \renewcommand{\arraystretch}{1.3} 
    \caption{Examples of inference rules used in this work} 
    \centering
    \begin{tabular}{| l | c |}
        \hline
        \multicolumn{1}{| c |} {Rule} & Description \\
        \hline
        $F_{G}(X) + F_{R}(Y, X) \Rightarrow F_{G}(Y)$ & global position of $X$ + relative position of $Y$ to $X$ $\Rightarrow$ global position of $X$\\
        \hline
        $F_{R}(Y, X) \Rightarrow F_{R}(X, Y)$ & relative position of $Y$ to $X$ $\Rightarrow$ relative position of $X$ to $Y$\\
        \hline
        $F_{R}(X, Z) + F_{R}(Y, Z) \Rightarrow F_{R}(X, Y)$  & relative position of $X$ to $Z$ + relative position of $Y$ to $Z$ $\Rightarrow$ relative position of $X$ to $Y$\\
         \hline
   \end{tabular}
    \label{tab:rules}
\end{table*}

Information instances are used to label the actual information. 
In this work, we use capital $F$ to denote information instance and type,
and $f$ to denote the {\textbf{\textit{value}}} of the actual information. 
For example, $F_R(r_1, r_2)$ is used  to refer to ``{\it the relative position between $r_1$ and $r_2$}'',
where the suffix $R$ denotes relative position;
$f_R(r_1, r_2)$ corresponds to a specific value of this information. 
For brevity, we often use $F$ without the referents to denote an information instance. 
Next, we more formally define physical constraint as follows:

\begin{definition}[Physical Constraint]
    A physical constraint is
     a constraint on the value of an information instance $F$. 
    \label{def:info-config}
    \qed
\end{definition}

Notice that the exact value for a constraint may depend on the environment settings and robot configurations dynamically, and hence is not always specified a priori.  
For example, in a tracking task, a constraint specified with respect to the target may be influenced by the environment settings (e.g., whether occlusions are likely to occur). 
This value is assumed to be determined by the execution module (which is not the focus of this work). 
To infer about information invariant, we define
information inference:

\begin{definition}[Inference Rules]
    Given a set of information instances, $S$, and an information instance $F$, 
    an inference rule defines a relationship such that any value set for $S$, i.e., $\{f_1: F_1 \in S\}$, uniquely determines the value of $F$ (i.e., $f$), or written as $S \Rightarrow F$.
    \qed
    \label{def:infer-iis} 
\end{definition}

For example, 
$\{F_R(r_1, r_2), F_G(r_2)\}$ (the relative position between $r_1$ and $r_2$ and the global position of $r_2$) can be used to infer $F_G(r_1)$ (the global position of $r_1$).
See Tab. \ref{tab:rules} for a few more examples.

\subsection{Problem Formulation}
\label{sec:coal-coord}

The setting of multi-tasking robots with multi-robot tasks (with instantaneous assignment) is given as follows:

\begin{definition}[MT-MR Setting]
A MT-MR setting is given by a set $\{S_i\}$,
where $S_i$ corresponds to the set of constraints to be satisfied by the coalition for task $t_i$, where the values of $S_i$ are determined by each coalition. 
\end{definition}

We further stipulate that the task requirements, manifested as physical constraints, are {\it independent} among themselves, which depend only on the environment, task and robot configurations. 
For example, in our motivating example (Fig. \ref{fig:demo}), the constraint introduced by the monitoring task is solely a requirement of the monitoring task, and independent of the centroid task. 
The problem of enabling MT-MR 
thus becomes that of determining whether there exists a physical configuration of the robots that satisfies all these constraints at the same time:

\begin{definition}[Compatibility]
A  set of constraints $S$ is compatible if 
there exists a physical configuration for all the referents in $S$ such that all the constraints are satisfied.
\end{definition}

To reason about compatibility, 
we first consider the inverse when a set of constraints is incompatible. 
Assuming that the physical constraints are independent and based on the definition above, the set of constraints becomes incompatible if two constraints associated with the same information instance are constrained by different values. 
We formally define this intuition in the following claim:

\begin{claim}
Given a MT-MR setting $\{S_i\}$, $\{S_i\}$ is incompatible if no two constraints are restricted by the same information with different values. 
\label{am:comp}
\end{claim}

While the above claim provides a sufficient condition for a set of constraints to be incompatible, 
it is not a necessary condition. 
For necessity to hold, however, a process must exist such 
that two constraints restricted by the same information with different values will {\it always} be found whenever a physical configuration does not exist for a given set of constraints. 
This requires us to reason about the {\it equivalences} between constraint systems, or in our case information systems, which are exactly what the information invariant theory~\cite{donald1995information} deals with!
The underlying assumption for realizing this process is the ability to identify all information instances that are relevant  and can be derived (or inferred) from known information in the domain.
To achieve this, we must first assume that a complete set of inference rules is identified in a domain, which is an implicit assumption made throughout the paper.
Next, we incrementally develop a mechanism to implement this process for our problem settings. 
First, we introduce a new concept:
\begin{definition}[Inference Closure]
   Denoting the inference closure of a set of information instances $S$ by $\mathcal{C}(S)$, any $F$ belongs to $\mathcal{C}(S)$ if:
    \begin{enumerate}
    	\item $F \in S$ or
        \item $\mathcal{C}(S) \Rightarrow F$
    \end{enumerate}
    \label{def:pow-iis}
\end{definition}

Note the recursive definition here. 
Given the inference rules in Tab. \ref{tab:rules}, for example,
we can conclude that $\mathcal{C}(\{F_G(r_1)$, $F_G(r_2)\}) 
= \{F_G(r_1)$, $F_G(r_2), F_R(r_1, r_2)$, $F_R$ $(r_2$, $r_1)\}$.
We also refer to any information instance that is in $\mathcal{C}(S)$ 
as {\it ``inferable''} from $S$, 
or that $S$ infers it. 
Based on this definition, $S$ trivially infers any instance already in $S$. 
We use $\rightarrow$ to distinguish it from that used in inference rules.
Note, however, that $\rightarrow$ subsumes $\Rightarrow$.
$\rightarrow$ is clearly transitive by definition.
Next, we more formally define the notion of information inference that links us to inference closure:

\begin{definition}[Information Inference]
    Given a set of information instances, $S$, and an information instance $F$, 
    an information inference defines a relationship such that any value set for $S$, i.e., $\{f_1: F_1 \in S\}$, uniquely determines the value of $F$ (i.e., $f$), or written as $S \rightarrow F$.
    \qed
    \label{def:infer-iis} 
\end{definition}

Note the similarity between inference rule and information inference. 
When $S_1 \leftrightarrow S_2$, we refer to them as being equivalent sets.
Intuitively, information inference enables us to reason about the constraints that are implicitly ``{\it required}'' as a result of the given set of constraints.
We show it more formally next:

\begin{lemma}
Given an information inference in the form of $S \rightarrow F$,
if a set of constraints is defined over $S$, it also introduces a constraint on $F$.
\label{lm:constraints}
\end{lemma}

This follows almost immediately from the definition. 
Given a set of values for $S$, the value of $F$ is determined from Def. \ref{def:infer-iis},
which implies that $F$ is also constrained according to Def. \ref{def:info-config}. 
A derivative of this result is that if $S_1 \rightarrow S_2$, 
a set of constraints on $S_1$ also introduces a set of constraints on $S_2$.

\begin{definition}[Minimally Sufficient Inference]
	$S \rightarrow F$ is a minimally sufficient inference if removing any instance from $S$ would no longer infer $F$.
    \label{def:min-suff}
\end{definition}
Any inference rule is always assumed to be a minimally sufficient inference in this work,
since otherwise the rule can be simplified by removing the instances that are not required. 
We use $\rightarrow^*$ to denote a minimally sufficient inference. 
In the following,
we further simplify our discussion by assuming linear information systems:

\begin{lemma}[Linear Information Systems]
Assuming all inference rules specify linear relationships among the information instances,
any information inference of the form $S \rightarrow F$ also specifies a linear relationship. 
\label{lm:linear}
\end{lemma}

\begin{proof}
Given that $S \rightarrow F$, 
there must exist a set of inference rules that are sequentially applied to derive $F$, 
in the forms of $S_1 \Rightarrow F_1, S_2 \Rightarrow F_2, \dots, S_k \Rightarrow F_k (F)$,
where $S_i \subseteq S \bigcup_{j < i} \{F_j\}$.
Since all the inference rules are assumed to be linear, 
we may replace $F_i$ appearing after the $i$th rule using the $i$th rule for expressing $F_i$, which removes $F_i$ from the equations.
After performing this operation sequentially from $i = 1$ to $k - 1$,
we end up with an expression of $F$ using only $S$. 
\end{proof}
The rules in Tab. \ref{tab:rules} define a linear information system. 

\begin{lemma}[Permutability]
Assuming a linear information system,
any minimally sufficient inference of the form $S \rightarrow F$ is permutable, or in other words it satisfies that $S \cup \{F\} \setminus F_x \rightarrow F_x$, which is also a minimally sufficient inference, for all $F_x \in S$. 
\label{lm:perm}
\end{lemma}

\begin{proof}
Given that $S \rightarrow^* F$ specifies a linear relationship (Lemma \ref{lm:linear}), the linear expression of $F$ using $S$ as constructed in Lemma \ref{lm:linear} must contain all the instances in $S$ without any coefficients being zero. 
Given a linear relationship, 
we can swap the positions of any $F_x$ and $F$ in the expression, and the result is still a valid linear equation for expressing $F_x$. 
Since $F_x$ expressed by this equation is uniquely determined by $S \cup \{F\} \setminus F_x$ collectively only, 
by Def. \ref{def:min-suff} we have $S \cup \{F\} \setminus F_x \rightarrow^* F_x$.
\end{proof}

\begin{lemma}
Assuming linear information systems and $S_1$ and $S_2$ are two sets of constraints in a MT-MR setting:  if
$S_1 \rightarrow F$ and $S_2 \rightarrow F$ are both minimally sufficient inferences, and $S_1$ and $S_2$ are compatible at the same time, then we must have
$S_1 = S_2$.
\label{lm:complete}
\end{lemma}

\begin{proof}
We prove it by contradiction. Suppose that $S_1$ and $S_2$ are compatible at the same time and $S_1 \not= S_2$. There must exist
a set of values for $S_1$, which  corresponds to a physical configuration of the referents that also satisfies $S_2$. 
Given that $S_1 \rightarrow F$, we know that $S_1$ introduces a constraint on $F$ given Lemma \ref{lm:constraints}.
Hence, to ensure that $S_2$ is compatible, $S_2$ must compute the same value $f$ for $F$. 

If $S_1$ and $S_2$ are the same, the conclusion trivially holds. 
Otherwise, 
when $S_1 \supset S_2$, it results in a contradiction given that both $S_1 \rightarrow F$ and $S_2 \rightarrow F$ are minimally sufficient.
Otherwise, $S_2$ must contain at least one instance $F_{2, \neg 1}$ that is not present in $S_1$. 
In such a case, our problem setting (MT-MR) has the flexibility to set the value in $S_2$ independently of $S_1$.
Since updating the value of $F_{2, \neg 1}$ will change the value of $F$ given Lemma \ref{lm:perm}, it leads to a contradiction that $S_1$ and $S_2$ must compute the same value for $F$.
\end{proof}

\begin{figure}
    \centering
    \includegraphics[width=0.8\columnwidth]{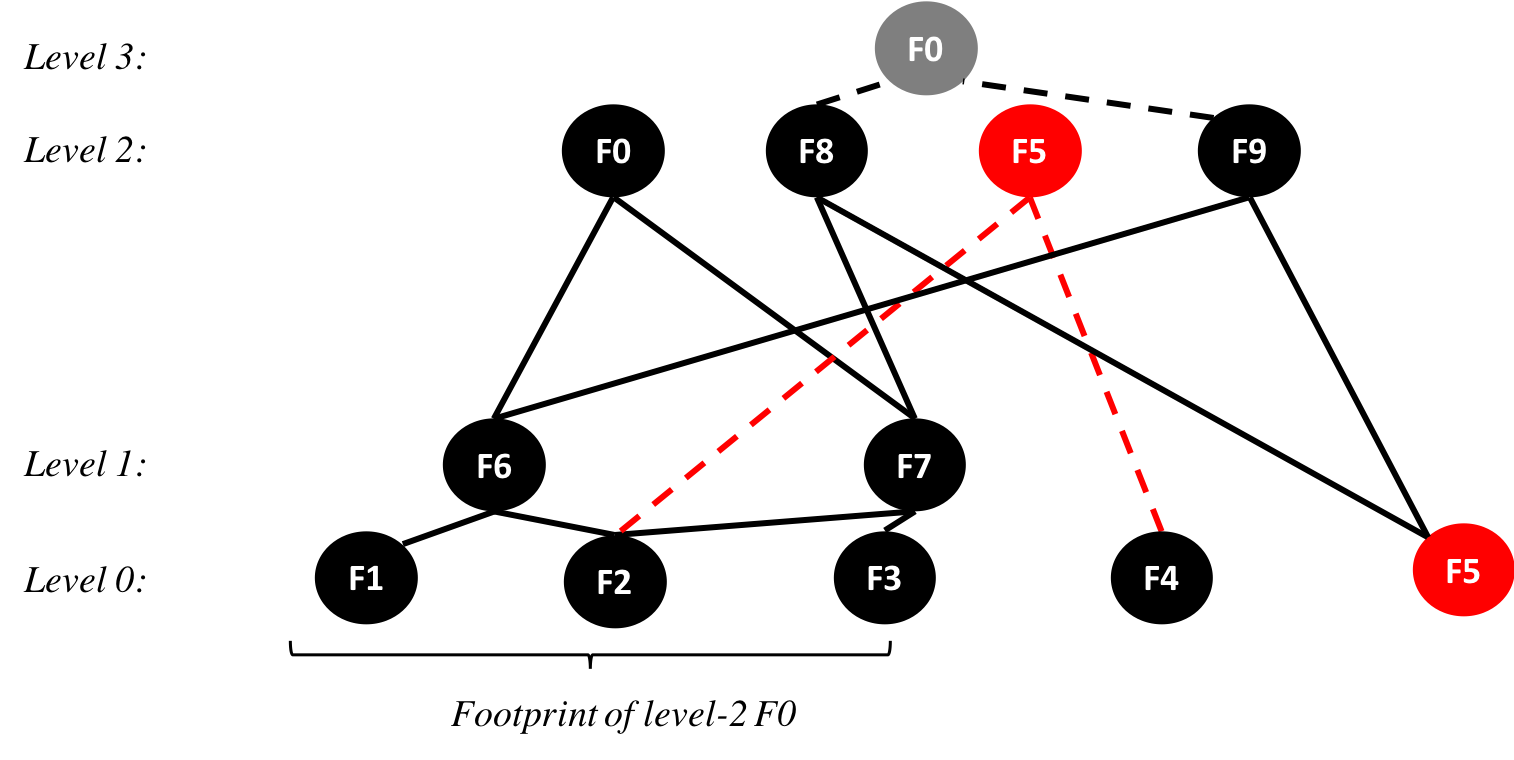}
    \caption{Illustration of the multi-level graphical structure constructed (from the bottom up) to determine whether a MT-MR setting is compatible. The nodes are labeled by information instances. It shows two cases where a duplicate node is found: one for $F_5$ and one for $F_0$. The $F_0$ at level $3$ is not added and it does not lead to incompatibility. The $F_5$ at level $2$ however leads to incompatibility. In the actual implementation, our algorithm would already stop at level $2$ after the incompatibility is detected.}
    \vskip-15pt
    \label{fig:layer}
\end{figure}

Finally, we state the main theorem in the following:

\begin{theorem}
Assuming linear information systems:
given a MT-MR setting $\{S_i\}$ with non-overlapping $S_i$'s,
the following is a {\textbf{necessary and sufficient}} condition for $\{S_i\}$ to be incompatible:
\begin{equation}
 S_1 \cap S_2 \not\rightarrow F,  S_1 \rightarrow F, S_2 \rightarrow F \text{ where } S_1, S_2 \subseteq \bigcup_i S_i
\label{eqn:complete}
\end{equation}
\label{thm:complete}
\end{theorem}
\vskip-10pt

\begin{proof}
For sufficiency, we prove it by contradiction. 
In particular, we assume that there exist $S_1$ and $S_2$ that satisfy the above conditions and they are compatible. 
Given that $S_1 \rightarrow F, S_2 \rightarrow F$, 
we know that there exist subsets $S_1^*$ and $S_2^*$ of $S_1$ and $S_2$, respectively, such that $S_1^* \rightarrow^* F, S_2^* \rightarrow^* F$.
From  Lemma \ref{lm:complete}, we know that for them to be compatible (as a result of $S_1$ and $S_2$ being compatible), it must satisfy that $S_1^* = S_2^*$. 
This conflicts with the fact that $S_1 \cap S_2 \not\rightarrow F$.

For necessity, we must prove that the above conditions must hold for all 
situations where $\{S_i\}$ is not compatible. 
Assume a situation where the conditions do {\it not} hold and $\{S_i\}$ is incompatible. 
In such a case, there must exist two different sets $S_a \subseteq \bigcup_i S_i$ and $S_b \subseteq \bigcup_i S_i$ that infer $F$ (given Claim \ref{am:comp} and the assumption about a complete set of inference rules), 
and that  $S_a \cap S_b \rightarrow F$ (given the assumption above).
In which case, however, $S_a$ and $S_b$ must be compatible for $F$, 
resulting in a contradiction. 
\end{proof}

Note that when $S_i$'s overlap, they are trivially incompatible by Claim \ref{am:comp}. 

\subsection{Solution Method}
Brute-forcing the solution is clearly intractable since it requires checking all subset-pairs of $\bigcup_i S_i$, 
which is exponential. Instead, we propose the following procedure based on a directed and multi-level graphical structure constructed from the bottom up:

\begin{figure*}[t!]
    \centering
    \includegraphics[width=0.6\columnwidth]{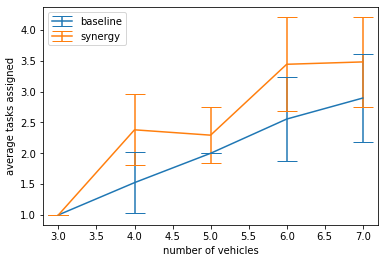}
    \includegraphics[width=0.6\columnwidth]{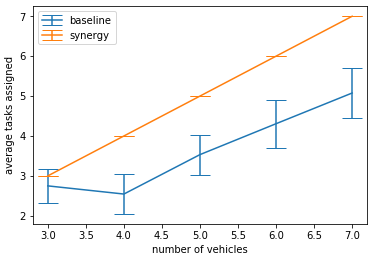}
    \includegraphics[width=0.6\columnwidth]{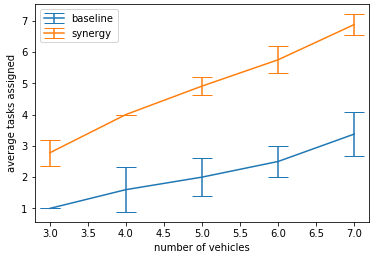}
    \caption{
    Plots of average tasks assigned as number of vehicles increases. Each plot specifies a different test configuration with fewer tasks of both types, more monitoring tasks, and more tasks of both types, with respect to the number of vehicles. The data for each plot was generated from 100 iterations, each given a random number of vehicles from $3$ to $7$.}
    \vskip-10pt
    \label{fig:1}
\end{figure*}

\begin{itemize}
    \item \textbf{\textit{Level 0}}: Make a node for each $F \in \bigcup_i S_i$ as leafs for the structure.
    \item \textbf{\textit{Level i + 1}} $(i \geq 0)$: For all inference rules that can be applied to the nodes at levels $0$ to $i$, make a parent node for each instance $F$ that can be inferred directly based on an inference rule, if this instance does not appear previously in the graph. Otherwise, we compute the intersection of its footprint (all descendant-leaf nodes, see Fig. \ref{fig:layer}) and that of the previous node, to see if it infers $F$. If so, continue with building the graph (without adding the duplicate node); otherwise, return incompatible.
    \item \textbf{\textit{Stopping criteria}}: when no new nodes can be created, return compatible.
\end{itemize}

Fig. \ref{fig:layer} shows an example of the graphical structure constructed to illustrate the compatibility detection process. 

\subsection{Solution Analysis}
To analyze the complexity of the algorithm, we consider the following variables:

\begin{itemize}
    \item number of inference rules, $R$
    \item number of agents, $G$
    \item number of information types, $F$
    \item maximum number of referents in instances, $E$
    \item maximum number of information instances on the left hand side of an inference rule, $N$
\end{itemize}

The maximum number of information instances is bounded by $O(FG^E)$.
At any level $i$, the number of candidate rules to check is bounded by $O(RG^{EN})$. 
Hence, the total computation for constructing the graph is bounded by $O(FRG^{EN})$. 
Hence, the computational complexity is only exponential in two constants (determined by the domain), and otherwise polynomial with respect to the number of agents. 

\begin{theorem}
The solution method  specified above is both sound and complete for detecting incompatibility in a MT-MR setting with linear information systems. 
\end{theorem}

This is a direct result of Theorem \ref{thm:complete} since the solution method essentially implements the same checking process described there. 
Note, however, when the system returns compatible, it does not necessarily mean that there exists a physical configuration in the current situation, since the environment may also affect the compatibility. 
However, assuming that the influence of the environment is temporary, a system should be able to recover from an incompatible state. 
Further discussion on this is delayed to future work. 

\section{Results}

\subsection{Tasks Considered}

We introduced three types of UAV tasks that were considered in our experimental results.

\begin{itemize}
    \item \textbf{Monitoring task}: a target must be monitored within a close proximity by an air vehicle. The constraint for the monitoring task is the relative position ($F_R$) between the vehicle and target, and  the global position ($F_G$) of the target (since we do not have control over it). 
    \item \textbf{Centroid task}: a group of vehicles must maintain their centroid with respect to a specific location or target. The constraint is the centroid information  defined over either $2$ or $3$ robots (denoted by $F_{C_2}$ or $F_{C_3}$). The centroid information can be derived (inferred) from the global position information of the vehicles involved in the centroid task. 
    \item \textbf{Communication maintenance task}: a vehicle must maintain its position in between two other vehicles to maintain communication links. The constraint introduced here is the communication maintenance information (denoted by $F_M$)  that takes $3$ vehicles, which can be derived from the relative positions between vehicles $1$ and $3$, and $2$ and $3$, assuming that $3$ is the vehicle that is between the other two.
\end{itemize}

\subsection{Synthetic evaluation} 

In this experiment, we tested with the first two types of tasks only. 
The goal is to see how beneficial our approach is under resource constrained situations. 
We ran several experiments to determine the efficacy of our synergistic approach compared to a baseline. In the baseline, vehicles were assumed to be single-tasking, and hence did not accept new tasks once they were assigned to a task. 
In our experiments, we randomly generated sets of tasks for specific sets of agents, and attempted to assign tasks alternately between centroid and monitoring tasks to vehicles,
until every generated task has been attempted. 

In Fig. \ref{fig:1}, we set out to evaluate how our approach performs as the number of vehicles increases from $3$ to $7$.
In the figures, we varied the ratio between the numbers of the centroid and monitoring tasks, which was expected to have a noticeable effect on the performance. 
The top figure shows a configuration where the numbers of both tasks are half of the number of vehicles. 
This evaluates situations where the robotic resources are more abundant. 
In such cases, the performance between the two methods did not differ much even though our synergy based method still outperformed the baseline.
The middle figures shows a configuration with many more monitoring (simple) tasks,
in which we observed more task assignments for both methods. 
This illustrates situations where resources are more constrained. 
Our synergy method not only performed significantly better, with almost 1.5 times tasks assigned, but also more consistent: the random task generation had much less effect on our method with almost zero standard deviations. 
In the bottom figure, we increased the numbers of both tasks for task-saturated situations. 
We can see that the influence on our synergy method was still less apparent than that on the baseline. 
In every configuration examined, our approach resulted in more tasks assigned than the baseline approach.

 In Fig. \ref{fig:2} (top), 
 we studied the differences of the assignment process between the two approaches in more detail. 
 We generated a new task at each iteration of the assignment process and checked how many tasks were allocated in each iteration accumulatively. We can see that at very low numbers of tasks relative to the vehicles, the synergy and baseline task assignments assigned the same number of tasks, which was to be expected since we had many more vehicles available. 
 However,  as more tasks were added to the assignment, the synergy method pulled ahead of the baseline. 
 Also, our method consistently assigned more tasks in each iteration. 
 In Fig. \ref{fig:2} (bottom), we studied the influence of the ratio between the numbers of the monitoring tasks and centroid tasks in detail. 
 We randomly generated $25$ tasks with the ratio between the tasks gradually decreasing, so that initially we had $25$ monitoring and $0$ centroid task, and $0$ monitoring and $25$ centroid task at the end. 
 We can see that our method was affected a lot less by the baseline. The performance gap kept increasing as this ratio decreased.


\begin{figure}[t!]
    \centering
    \includegraphics[width=0.6\columnwidth]{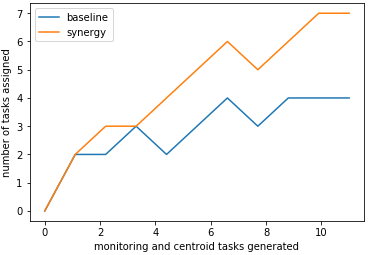}
    \includegraphics[width=0.6\columnwidth]{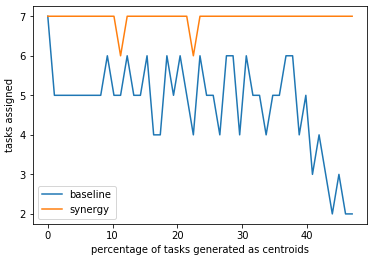}
    \caption{
    (top) Plot showing the number of tasks assigned by each approach as the number of tasks generated increases. (bottom) Plot showing the number of tasks assigned by each approach with $25$ tasks as the ratio between the monitoring and centroid tasks gradually decreases.}
    \vskip-10pt
    \label{fig:2}
\end{figure}



\begin{figure}
    \centering
    \begin{subfigure}
        \centering
        \includegraphics[width=0.48\columnwidth]{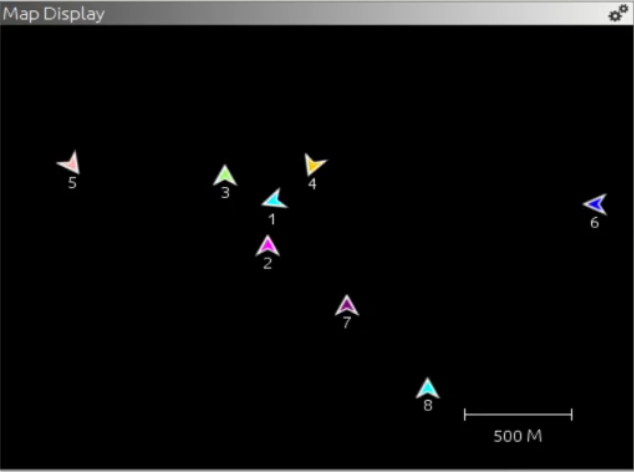}
    \end{subfigure}%
    ~ 
    \begin{subfigure}
        \centering
        \includegraphics[width=0.48\columnwidth]{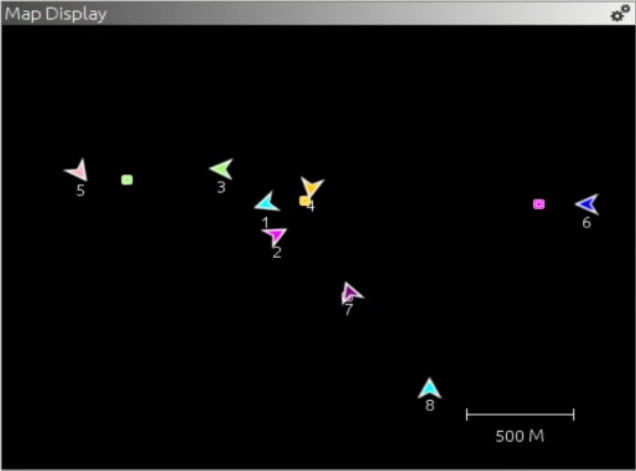}
    \end{subfigure}
    ~ 
    \begin{subfigure}
        \centering
        \includegraphics[width=0.475\columnwidth]{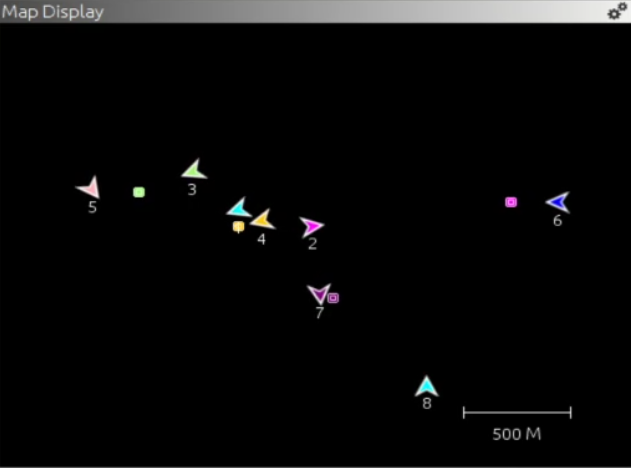}
    \end{subfigure}
    ~
        \begin{subfigure}
        \centering
        \includegraphics[width=0.471\columnwidth]{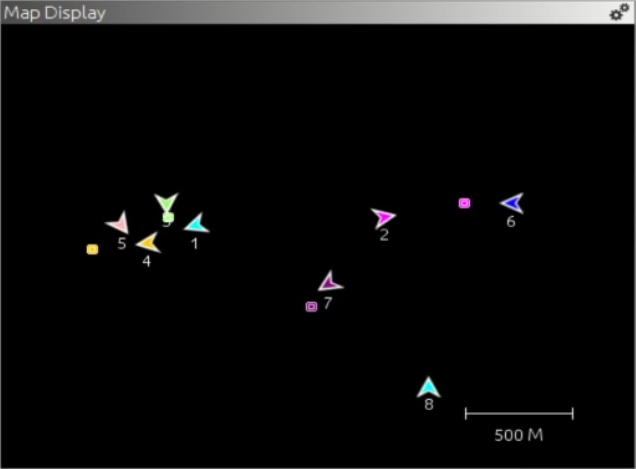}
    \end{subfigure}%
    ~ 
    \begin{subfigure}
        \centering
        \includegraphics[width=0.475\columnwidth]{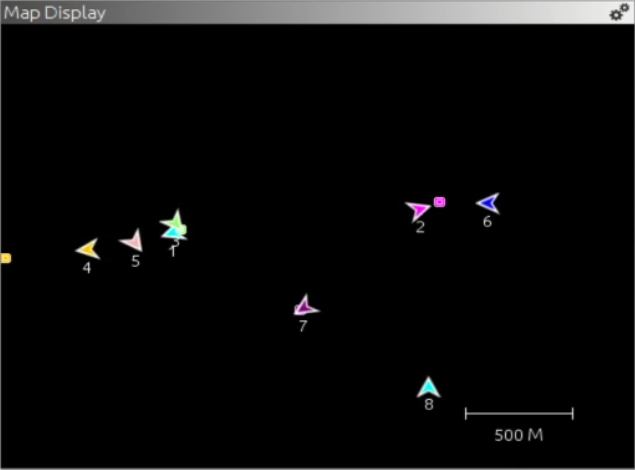}
    \end{subfigure}
    ~ 
    \begin{subfigure}
        \centering
        \includegraphics[width=0.47\columnwidth]{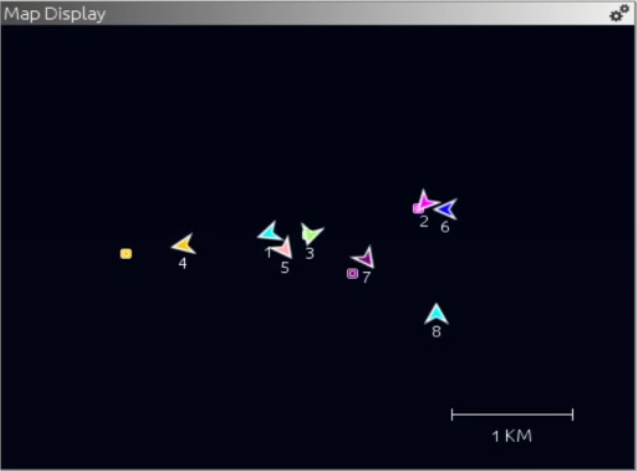}
    \end{subfigure}
     \vskip-5pt
    \caption{Screenshots of the simulated task, going from left to right, and top to bottom, as the simulation progresses. Each vehicle is represented by a chevron with a unique color and an ID label. Each controlled vehicle has its goal destination shown by a dot of its color. Vehicles 1, 5, 6, and 8 are not controlled. The first shows the initial setting, and the last figure shows the scenario in a different scale from others. 
    }
    \vskip-15pt
    \label{fig:4}
\end{figure}

\subsection{Simulation Scenario}

\subsubsection{Simulation Environment and Settings}
The OpenAMASE simulation environment was developed by the Air Force Research Lab\cite{afrlAMASE} as a testing ground for their aerial vehicle control software, UxAS\cite{afrlUxAS}. Together, these two pieces of software form the basis of the simulation environment. AMASE gives access to a GUI and simulates the vehicles over time, and UxAS handles the passing of all relevant messages to and from all modules of the software. Any module can subscribe to any type of message, and will then receive any message of that type sent by any other module. Our software uses the AirVehicleState, containing a ``heartbeat'' of information about each vehicle for each simulation tick, and the AirVehicleConfiguration, containing capability information about each vehicle, to decide where vehicles should move to.

\subsubsection{Simulation Result}
In this simulation, we tested our system on a realistic scenario involving a multi-vehicle convoy task with intruder detection. 
AMASE and UxAS can only handle up to twelve vehicles, and we have limited our simulation to $8$ controlled vehicles,  $2$ intruder vehicles, and $1$ static vehicle that simulates a control station. 
Fig. \ref{fig:4} shows snapshots from running the task. 
A convoy is formed (centered around vehicle $1$) and protected by three vehicles ($2, 3, 4$), which are assigned to a centroid task with vehicle $1$ as the centroid.
In the mean time, vehicle $7$ must maintain the communication between the convoy and a ground control station (vehicle $8$). 
As the convoy is moving towards its target, 
two intruders are detected and two of the vehicles ($2$ and $3$) that
are already assigned the centroid task take advantage of
the synergy between centroid and monitoring tasks by executing two tasks at the same time.

\section{Conclusions}
In this paper, we set out to address the problem of multi-tasking robots in multi-robot tasks. 
We observed that the key underlying challenge was to reason about the physical constraints that could be synergistically satisfied.
To address the challenge, we developed our method based on the information invariant theory and modeled constraints as information instances. 
Thereby, a formal and general framework to achieve multi-tasking robots was developed. 
We showed that our algorithm was sound and complete under our problem settings. 
Simulation  results  were  provided  to  show  the  effectiveness  of  our approach under resource-constrained situations and in handling challenging situations. 

\bibliographystyle{IEEEtran}
\bibliography{references}

\end{document}